\documentclass[conference]{IEEEtran}
\IEEEoverridecommandlockouts

\usepackage{cite}
\usepackage{amsmath,amssymb,amsfonts}
\usepackage{algorithmic}
\usepackage{graphicx}
\usepackage{textcomp}
\usepackage{xcolor}
\def\BibTeX{{\rm B\kern-.05em{\sc i\kern-.025em b}\kern-.08em
		T\kern-.1667em\lower.7ex\hbox{E}\kern-.125emX}}

\usepackage{url}
\usepackage[utf8]{inputenc}
\DeclareUnicodeCharacter{2010}{-}

\usepackage{upgreek}
\usepackage[utf8]{inputenc}
\usepackage[english]{babel}
\usepackage{booktabs} 
\usepackage{breqn}
\usepackage{subfig}
\newcommand*\circled[1]{\tikz[baseline=(char.base)]{
		\node[shape=circle,draw,inner sep=1pt] (char) {#1};}}
\usepackage{tablefootnote}
\usepackage{footnote}
\makesavenoteenv{tabular}
\usepackage{footmisc}

\usepackage{xspace}

\newcommand{\eg}{{\it e.g.,}\xspace}

\newcommand{\vs}{{\it vs.}\xspace}

\newcommand{\etal}{{\it et~al.}\xspace}
\newcommand{\ie}{{\it i.e.,}\xspace}
\newcommand{\etc}{{\it etc.}\xspace}

\newcommand{\ca}{{\it (a) }}
\newcommand{\cb}{{\it (b) }}

\usepackage{amsthm}
\theoremstyle{definition}

\newtheorem{corollary}{Corollary}

\newtheorem{theorem}{Theorem}

\newcommand{\lidar}{LiDAR\xspace}

\usepackage{scalerel}
\usepackage{tikz}
\usetikzlibrary{svg.path}
\definecolor{orcidlogocol}{HTML}{A6CE39}
\tikzset{
  orcidlogo/.pic={
    \fill[orcidlogocol] svg{M256,128c0,70.7-57.3,128-128,128C57.3,256,0,198.7,0,128C0,57.3,57.3,0,128,0C198.7,0,256,57.3,256,128z};
    \fill[white] svg{M86.3,186.2H70.9V79.1h15.4v48.4V186.2z}
                 svg{M108.9,79.1h41.6c39.6,0,57,28.3,57,53.6c0,27.5-21.5,53.6-56.8,53.6h-41.8V79.1z M124.3,172.4h24.5c34.9,0,42.9-26.5,42.9-39.7c0-21.5-13.7-39.7-43.7-39.7h-23.7V172.4z}
                 svg{M88.7,56.8c0,5.5-4.5,10.1-10.1,10.1c-5.6,0-10.1-4.6-10.1-10.1c0-5.6,4.5-10.1,10.1-10.1C84.2,46.7,88.7,51.3,88.7,56.8z};
  }
}
\newcommand\orcidicon[1]{\href{https://orcid.org/#1}{\mbox{\scalerel*{
\begin{tikzpicture}[yscale=-1,transform shape]
\pic{orcidlogo};
\end{tikzpicture}
}{|}}}}
\usepackage[colorlinks,linkcolor=black,citecolor=black,urlcolor=black,bookmarks=false,hypertexnames=true]{hyperref}

\newcommand\copyrighttext{%
	\footnotesize
	\copyright 2022 IEEE. Personal use of this material is permitted. Permission from IEEE must be obtained for all other uses, in any current or future media, including reprinting/republishing this material for advertising or promotional purposes, creating new collective works, for resale or redistribution to servers or lists, or reuse of any copyrighted component of this work in other works.	
}
\newcommand\copyrightnotice{%
	\begin{tikzpicture}[remember picture,overlay]
		\node[anchor=south,yshift=10pt] at (current page.south) {\fbox{\parbox{\dimexpr\textwidth-\fboxsep-\fboxrule\relax}{\copyrighttext}}};
	\end{tikzpicture}%
}%

\begin{document}

\title{Verifiable Obstacle Detection}

\author{
\IEEEauthorblockN{
   Ayoosh Bansal%
   \IEEEauthorrefmark{1}%
   \orcidicon{0000-0002-4848-6850},
   Hunmin Kim%
   \IEEEauthorrefmark{2},
	 Simon Yu%
	 \IEEEauthorrefmark{1}%
	 \orcidicon{0000-0002-9845-6983},
    Bo Li%
    \IEEEauthorrefmark{1},
    Naira Hovakimyan%
   \IEEEauthorrefmark{1},
   Marco Caccamo%
   \IEEEauthorrefmark{3},
   Lui Sha%
   \IEEEauthorrefmark{1}
}

\IEEEauthorblockA{
   \IEEEauthorrefmark{1}University of Illinois Urbana-Champaign
   \{ayooshb2, jundayu, lbo, nhovakim, lrs\}@illinois.edu,
}
\IEEEauthorblockA{
   \IEEEauthorrefmark{2}Mercer University
    kim\_h@mercer.edu,
   \IEEEauthorrefmark{3}Technical University of Munich
   mcaccamo@tum.de
}
}

\maketitle

\copyrightnotice

\begin{abstract}

Perception of obstacles remains a critical safety concern for autonomous vehicles.
Real-world collisions have shown that the autonomy faults leading to fatal collisions originate from obstacle existence detection.
Open source autonomous driving implementations show a perception pipeline with complex interdependent Deep Neural Networks.
These networks are not fully verifiable, making them unsuitable for safety-critical tasks.

In this work, we present a safety verification of an existing \lidar based classical obstacle detection algorithm.
We establish strict bounds on the capabilities of this obstacle detection algorithm.
Given safety standards, such bounds allow for determining \lidar sensor properties that would reliably satisfy the standards.
Such analysis has as yet been unattainable for neural network based perception systems.
We provide a rigorous analysis of the obstacle detection system with empirical results based on real-world sensor data.

\end{abstract}


\section{Introduction}
\label{sec:introduction}

Autonomous Vehicles (AV) will be among the most significant technological achievements of current times, with the potential to save and improve lives~\cite{av_imp_1, av_imp_2}.
However, it is unclear at what point AV will be safer than the human-driven vehicles~\cite{dirty_dozen}.
The imperfection of current AV is evident in the various crashes that have been attributed, in part, to the autonomous control of involved vehicles~\cite{tesla_2, ubercrash, tesla_3, tesla_4, tesla_1, tesla_1_1, tesla_japan, boudette2021happened, tesla_5, tesla_6, tesla_7, tesla_8, tesla_9, tesla_10, tesla_11, tesla_12}.

Neural networks and artificial intelligence have enabled capabilities in cyber-physical systems that might have remained unachievable otherwise.
Deep Neural Networks (DNN) for object detection, classification, and predictive tracking, are crucial for perceiving an AV's environment in real-time and planning complex maneuvers an AV must execute.
But these technologies are inherently unverifiable, \ie incapable of being verified~\cite{koopman2016challenges, safety2016ieee, easy_fool_ai} that leads to, as yet unsolvable, safety concerns~\cite{koopman2018heavy,biggio2013evasion, amodei2016concrete, huang2017safety, tian2018deeptest, li2020deepdyve, willers2020safety}.
Safety-critical software are required to be analyzable and verifiable~\cite{feiler2013four,heimdahl2016software}, a role DNN solutions are not yet ready to fulfill~\cite{faa2016adaptive,jenn2020identifying,pereira2020challenges}.

The impasse caused by the necessity of DNN solutions to enable AV's mission capabilities and their unsuitability for safety-critical tasks can be resolved by decoupling the mission and safety requirements.
The disassociated fulfillment of safety and mission requirements has been successful in
system architectures like Simplex~\cite{simplex_early,simplex_original}, leading to systems that work with high performance in typical cases but provide verifiable safe behavior in safety-critical conditions.
But thus far, such a separation has not been achieved in AV.

In this work, we lay the foundation for decoupling of safety and mission responsibilities for AV.
We describe a minimal set of requirements for safety-critical perception with a focus on obstacle existence detection in Section~\ref{sec:minimal_ca}.
Unlike the mission-critical requirement of perceiving obstacles and predicting their trajectories for complex maneuvers planning,  the safety-critical requirement for collision avoidance is to reliably detect the existence of obstacles that the AV may collide with.
The safety-critical obstacle existence detection can then be used in various ways, like to detect and correct faults in the mission critical perception systems, fused as an ensemble~\cite{wei2018fusion,gurel2021knowledge}, or make control decisions like brake to stop if all else fails.

In the absence of verifiable DNN based solutions, we turn to classical obstacle detection algorithms, that are verifiable, \ie capable of being fully analyzed and verified, to determine their ability to fulfill the safety-critical requirements.
Traditional \lidar based obstacle detection algorithms~\cite{wang2022review}, based on geometric properties, are excellent candidates for verifiable safety-critical obstacle detection.
\lidar sensors have shown incredible promises for their use in autonomous driving~\cite{li2020lidar}.
This has, in turn, accelerated the improvements in \lidar sensor technology, increasing the range, scanning frequency, and resolution of the sensor~\cite{roriz2021automotive}.
\lidar sensors are also superior to stereo cameras for 3D obstacle detection as they only suffer a linear error growth with distance as compared to quadratic for stereo camera~\cite{wang2019pseudo}.
These factors made the \lidar sensor the natural choice for this work.

Depth Clustering~\cite{bogoslavskyi16iros, bogoslavskyi17pfg} a \lidar based range image segmentation algorithm is chosen as the example algorithm in this work.
The core component of the algorithm analyzed here is the ground removal \ie separating the flat drivable ground from obstacles that need to be avoided.
These bounds are referred to as the \textit{Detectability Model} in this work.
Such bounds satisfy the analyzability and verifiability requirements in safety-critical system engineering.
The value of such a Detectability Model comes from the predictability of faults in obstacle existence detection and reliable mitigation.
Given desired safety standards and properties of obstacles, algorithm parameters and \lidar sensor parameters can be chosen to meet the safety standards reliably.

The \textbf{contributions} of this work can be summarized as:
\begin{itemize}
	\item Minimized sufficient requirements for safety-critical obstacle existence detection for collision avoidance ($\S$\ref{sec:minimal_ca}).
	\item Detectability model for an existing classical obstacle detection algorithm, Depth Clustering, with human perceptible bounds on capabilities and limitations ($\S$\ref{sec:model}).
	\item An evaluation\footnote{\url{https://github.com/CPS-IL/verifiable-OD}} based on real sensor \mbox{data~\cite{sun2020scalability} ($\S$\ref{sec:eval}).}
\end{itemize}

\section{Motivation and Overview}
\label{sec:motivation}

Obstacle existence detection fault, \ie False Negative (FN) in perception are a grave safety concern~\cite{yang2021introspective}.
A survey of fatal collisions involving AV (Table~\ref{tab:crash_survey}) points to recurring FN errors.
Other fatalities involving AV~\cite{tesla_1,tesla_1_1,tesla_japan,boudette2021happened,tesla_5,tesla_6,tesla_7,tesla_8,tesla_9,tesla_10,tesla_11,tesla_12}, excluded from Table~\ref{tab:crash_survey} due to unavailability of investigation reports, seem to follow this pattern.
These underlying safety concerns are the primary challenge in adopting complete autonomy~\cite{nhtsa_investigation,kalra2017challenges,penmetsa2021effects,li2022safety,r2022liability,jing4011917listen}.
Learning from these incidents and acknowledging the impossibility of all encompassing safety on the road, we limit our focus to FN.

DNN verification, even for simple properties, is an NP-complete problem~\cite{katz2017reluplex}.
Gharib \etal~\cite{gharib2018safety} describe the need and current lack of verification methods for machine learning components used in safety-critical applications.
Empirical risk minimization, a foundational principle of modern statistical machine learning, fails to satisfy the high robustness requirements of safety-critical applications~\cite{varshney2016engineering}.
There is still an open requirement for verification techniques that can validate the behavior of a trained network under all circumstances and not just expected safe input space~\cite{urban2021review}.
The vastness of the input sets for real-world problems, like perception, renders reachability analysis impractical for the DNN used in AV.

Analyzability and verifiability are the crucial components of the certification process of safety-critical \mbox{systems~\cite{feiler2013four,heimdahl2016software}}.
Verifiable algorithms, where the \textit{causality} between the input parameters and the algorithm result can be established, are inherently suitable for safety-critical applications.
This is in contrast to object detection DNN, trained using supervised learning, which effectively captures \textit{correlations} between training input and labels.
Consider the following simple example where $y$ is obstacle height, $x$ is obstacle distance from the AV, $a$ and $b$ are constant parameters based on \lidar properties:
\begin{equation}
    y \geq ax + b \label{eq:ideal}
\end{equation}
If we want to establish that \eqref{eq:ideal} is the \textit{detectability model}, \ie when the condition in \eqref{eq:ideal} is met, obstacles are always detected by the AV,
the following must be defined:

{\it \textbf{Requirements}}: A definition of minimal requirements for what it means to \textit{detect} an obstacle ($\S$\ref{sec:minimal_ca}).

{\it \textbf{Constraints}}: A set of well defined constraints that must be met for the model to be applicable ($\S$\ref{sec:paramconstraints}).

{\it \textbf{Verification}}: A deterministic analysis verifying the detectability model ($\S$\ref{sec:model}).

Let's assume that an AV safety standard requires that the AV be able to detect all obstacles of a minimum height of $10$~$cm$.
Further, let's assume all vehicles and structures on the road are also mandated to be taller than this minimum height by a road safety rule.
Using \lidar and AV parameters to determine $a$ and $b$, the detectability model \eqref{eq:ideal} can be used to determine the minimum distance at which such obstacles can be detected.
This minimum distance, in conjunction with the braking capability of the AV, can then be used to determine the max speed at which the AV can safely travel.
This example, admittedly simple, shows how a verified detectability model can bring together road safety rules, AV parameters and AV safety policies to provide deterministic collision safety.

\begin{table}[t]
    \centering
    \caption{\label{tab:crash_survey}
        Survey of AV Involved Fatalities}
    \begin{tabular}{p{0.05\columnwidth} | p{0.85\columnwidth}}
        \toprule
        Ref. 		&   Autonomy Response
                \tablefootnote{The comments brief the driver assist system's response during the incident and are not necessarily the causal faults for the incident.
                    }
       \\ \midrule


        \cite{tesla_2} & A truck in the path of the vehicle was not detected and no evasive actions like braking or steering away were initiated.
        \\ \midrule

        \cite{ubercrash} & Low confidence, unstable classifications of a pedestrian led to the perception system ignoring the existence of a pedestrian.
        \\ \midrule

        \cite{tesla_3} & A crash attenuator in the path of the vehicle was not detected.
        \\ \midrule


        \cite{tesla_4} & A white semi trailer in the path of the vehicle was not detected.\\



        \bottomrule
    \end{tabular}
\end{table}

\section{Related Work}
\label{sec:relatedwork}

\textit{Classical Obstacle Detection:}
An approach for long-range obstacle detection based on stereo cameras was proposed by Pinggera~\etal~\cite{pinggera2015high}.
While they successfully detect patches on most objects within a long range, it is unclear whether this approach is enough to avoid collisions and whether all obstacles that pose collision risk are detected.
While the approach shows promise, its performance under common distortions like bright spots, \etc, is not shown, and a further study on its limitations is required to use it in safety-critical tasks.
Various \lidar based geometric algorithms were considered as part of this work~\cite{himmelsbach2010fast, korchev2013real, asvadi2015detection, chu2017fast, zermas2017fast}.
Each algorithm has a similar flow; identifying points on the ground vs. obstacles, followed by clustering, segmentation, and/or classification.
Depth Clustering~\cite{bogoslavskyi16iros, bogoslavskyi17pfg} is chosen as the primary example due to its deterministic explainable behavior, flexible parameters to optimize tradeoffs, and public availability of an efficient C++ implementation with extremely low runtime of $40$ $ms$ on an embedded Jetson Xavier platform, using a single CPU core only, for a point cloud with $169,600$ points~\cite{chen2021lidar}.
The algorithm has also garnered interest in recent works in literature~\cite{liu2020removing,chen2021lidar,dang2021sensor,shen2021fast,tian2021unsupervised,vobecky2022drive}.

\textit{Neural Network Verification:}
Albarghouthi~\cite{albarghouthi2021introduction} described the challenges for neural network verification, including scale, complexity, and dynamism of the environment, all applicable for their use in Autonomous Driving.
Liu~\etal~\cite{liu2019algorithms} survey existing verification methods and classify the verification methods into three categories: reachability, optimization, and search. They identify tradeoffs between scale and completeness of the verification methods.
Even when the deep networks can be verified, it is done so for small input sets only~\cite{tran2020verification}.
Another survey on the safety and trustworthiness of DNNs~\cite{huang2020survey} identified various challenges, including; a physical representation of the verification metrics, completeness of the verifiable properties, and scalability to complex DNN.
Verification techniques have also been proposed to explore around available inputs by adding adversarial or context specific distortions~\cite{huang2017safety, tian2018deeptest}, improving the input coverage.
However, this does not imply complete verification and dependable predictability of behavior.
Hardware reliability~\cite{li2020deepdyve} does not protect against algorithmic faults.
Fully verifiable, analyzable and explainable DNNs performing real-world object detection in autonomous vehicles remain an elusive goal.

\textit{Collision Avoidance:}
This work complements typical collision avoidance systems by potentially providing
additional triggers for the emergency braking systems to engage~\cite{seiler1998development,funke2016collision,cheng2019longitudinal}, \eg when comparing the output from the verifiable algorithm to the DNN, a safety-critical FN in DNN output is determined.
Other collision avoidance systems leverage cooperative communication~\cite{lin2014active}, however, our work focuses on all obstacles including human driven vehicles and pedestrians.
Motion planning~\cite{wang2019crash,lee2019collision,tahir2019heuristic} and risk assessment~\cite{noh2017decision,noh2018decision,shin2018human,yu2019occlusion,li2021risk} based collision avoidance systems would be benefited by using a verifiable algorithm to detect obstacles in the environment.

\section{Minimal Requirements}
\label{sec:minimal_ca}

\begin{figure*}[!t]
    \centering
    \includegraphics[width=\linewidth,keepaspectratio]{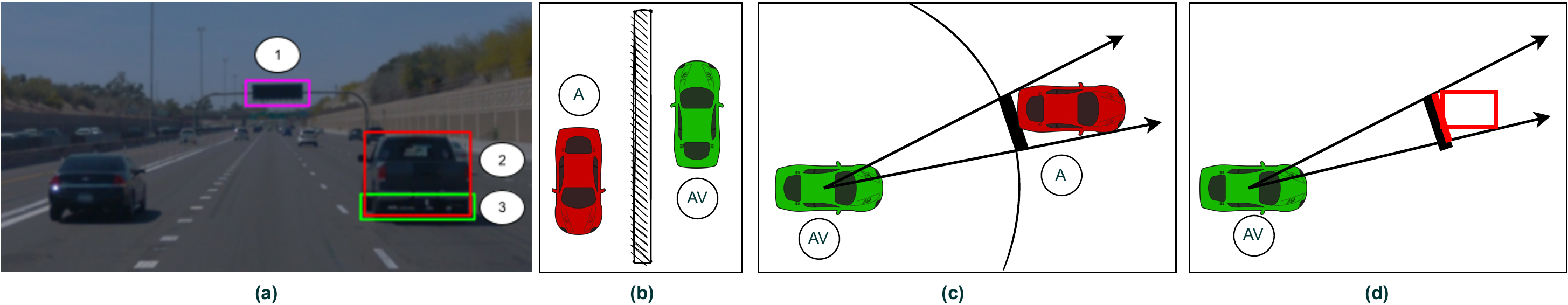}
    \caption{\label{fig:req_examples} Example scenes. (a) Role of obstacle height. (b) Distance and Occlusion of obstacles. (c) Projection of Obstacles.
    }
\end{figure*}

In this section, we define the minimally sufficient, though not necessary and sufficient, requirements for safety-critical obstacle detection and collision avoidance.
These requirements define a minimized set of features of an obstacle that the AV perception system should perceive to establish the existence of the obstacle.
While some of the following observations are already well established, this work is the first to form a minimally sufficient set for safety-critical collision avoidance.
Further, some requirements like object classification are typically considered crucial for all aspects of perception in AV, however, we show that while it is valuable for mission-critical requirements, it is not for safety-critical collision avoidance.

\subsection{Classification}
\label{sec:classification_not_required}

We will use the mathematical problem formulation of reach-avoid control to justify why we do not need to classify obstacles for collision avoidance purposes.
The current section is motivated by~\cite{summers2010verification,summers2011stochastic} but we simplify the notations and the problem in a deterministic setting while the references consider stochastic problem formulations.
Let  $x_k \in X$ be the 2D-position of the vehicle at discrete time instance $k \geq 0$, where $X \subseteq {\mathbb{R}}^2$ is the set of the state space. The set $D \subset X$ is the destination, and the sets $O_i(k) \subset X$ for $i \in M_O \triangleq \{1,2,\cdots,M\}$ are obstacles, where $O_i(k)$ represents the $i^{th}$ obstacle and $M_O$ is the index set for the obstacles. The set $O_i(k)$ could be time-varying but it is assumed that the obstacles do not overlap with the destination, \ie $\cup_{i \in M_O} (D \cap O_i(k)) = \emptyset$ $\forall$ $k$. For brevity, further mentions of $O_i(k)$ are reduced to $O_i$.
Now, the reach-avoid control problem is to drive the vehicle to the destination $D$ while avoiding obstacles $O_i$ for $i \in M_O$ within finite time horizon $N$, given initial condition $x_0$.
The success of this mission can be characterized by the following index $r$:
\begin{align}
	r& \triangleq \textstyle \sum_{j=0}^N \big(\Pi_{i=1}^M \Pi_{t=0}^{j-1} \bold{1}_{X \setminus O_i}(x_t)\big) \bold{1}_D(x_j)\nonumber\\
	&=
	\left\{
	\begin{array}{cc}
		1, &  {\rm if \ } \exists j \in [0, N]: x_j \in D \wedge \\
		&\forall t \in [0,j-1]: x_t \in \cap_{i=1}^M X \setminus O_i\\
		0, & {\rm otherwise,}\\
	\end{array}
	\right.
	\label{eq:formulation0}
\end{align}
where $\bold{1}_S(\cdot): X \rightarrow \{0,1\}$ is the indicator function for a set $S$, and $\wedge$ is the logical AND. In short, $r=1$ if and only if the objective is achieved. The index $r$ in~\eqref{eq:formulation0} can be used as a cost function of the optimal control problem as in~\cite{summers2010verification,summers2011stochastic}. Therefore, one is required to evaluate the indicator functions in~\eqref{eq:formulation0}, which means that it is required to know the set $D$ and $X \setminus O_i$ for $i \in M_O$.

However, index $r$ in~\eqref{eq:formulation0} can be equivalently formulated as
\begin{align}
	r& = \textstyle \sum_{j=0}^N \big( \Pi_{t=0}^{j-1} \bold{1}_{X \setminus (\cup_{i \in M_O} O_i)}(x_t)\big) \bold{1}_D(x_j)
	\label{eq:formulation1}
\end{align}
and this expression can be used instead for the optimal reach-avoid control problem.
The formulation~\eqref{eq:formulation1} indicates that one could also address the control problem only with knowing $\cup_{i \in M_O} O_i$, but not individual obstacle sets $O_i$, \ie
one does not need to classify/distinguish individual obstacles for the reach-avoid control purpose.

It should be noted that classification adds valuable information that supports advanced features like predictive tracking, maneuver planning and improves AV performance.
The argument here is only that object classification is not a necessity for obstacle avoidance and
therefore not a part of minimal requirements for safety critical obstacle detection.
\subsection{Collision Risk}
\label{sub:rrr}

Collision avoidance involves many factors, from perception to vehicular control.
While the dynamics and ethics of collision avoidance are complex~\cite{goodall2014ethical}, the safety-critical obstacle detection system is required to detect all obstacles that can potentially collide with the AV.
We utilize a physical model for collision risk from our prior work~\cite{9460196}, where
the potential risk of collision is determined by the overlap of the existence regions~\cite{schmidt2006research} of obstacles and AV within the AV's time to stop.

\subsection{Height}

The height of an obstacle is only useful in making a binary determination for collision avoidance, \ie
whether or not the obstacle is completely clear above the height of the AV.
For example, in Figure~\ref{fig:req_examples}~(a)  the overhead road sign \circled{1}  is completely above the AV, its exact height has no implication for collision avoidance.
The box \circled{2} contains valuable information that is required to identify the obstacle within the box to be a vehicle, however as discussed in Section~\ref{sec:classification_not_required}, such a recognition of the obstacle class is not a requirement for collision avoidance.
Thus for safety-critical collision avoidance \circled{3} contains as much relevant information as \circled{2}.
Therefore as long as an obstacle's height is not erroneously detected to be above and clear of the AV, we can simply consider the top or bird's eye view of the AV's surroundings.
While such a view of obstacles is not traditionally used in perception systems, however, path planning in AV, an inherently 2D problem, regularly uses this representation~\cite{ming2021survey}.
\subsection{Distance}
\label{sec:req_dist}

The distance to an obstacle must be accurately detected for the collision-free operation of the AV.
For collision avoidance, this distance is the minimum distance between the perimeters of the obstacle and the AV.
Many safety parameters like safe following distance, time to collision and time to stop, are a function of the distance between the AV and obstacles~\cite{seiler1998development, cheng2019longitudinal},
and therein lies the difference in underestimation and overestimation error in distance detection.

\subsubsection{Underestimation Error}
When obstacle distances are underestimated, \ie obstacles are detected to be closer than they are, the performance of AV suffers,
however, there is no negative impact for direct collision avoidance.
An implicit effect of acceptance of all detection with a lower distance than the ground truth obstacle is that occluded obstacles become unnecessary to detect as long as the occluding obstacle is detected.
An example for this is Figure~\ref{fig:req_examples} (a) and (b).
A small wall exists to the left side of the AV in both images, separating traffic moving in the opposite direction to the AV.
If the wall is detected as an obstacle, the AV would not need to detect the vehicles across the wall to avoid colliding with them. Avoiding collisions with the occluding obstacle (wall here) means avoiding collisions with occluded obstacles.
Note that this only applies to completely occluded obstacles. Partial occlusion is discussed in Section~\ref{sec:req_coverage}.

\subsubsection{Overestimation Error}
\label{sub:dist_error_positive}
Overestimation of distance is a grave safety concern.
Obstacles detected to be further than they are invalidate any safety decisions that would be based on this information.
Since it is unrealistic that an obstacle detection system would always have no distance overestimation error,
a strict bound on overestimation distance detection error is required.
With an established max error bound, any distance-based safety guarantee can assume this error is always present, maintaining the safety guarantee in the worst case.

\subsection{Projection}
\label{sec:req_coverage}

We determine a projection signifying a line on the road that the AV cannot cross without colliding with the obstacle for each obstacle.
Figure~\ref{fig:req_examples} (c) shows an example of this.
A circle is drawn with its center at the AV's sensor hub and radius equal to the closest point on the obstacle from AV.
The projection of the obstacle is then found as a line segment, tangent to the point where the above circle touches the obstacle.
This line segment is shown as a thick line segment in Figure~\ref{fig:req_examples} (c).
This line segment is a representation of the obstacle. Note that as new sensor inputs come in over time, the AV and obstacle move relative to each other, and the projection moves accordingly.
So the line segment only represents the obstacle in the current frame to make safety-critical decisions until the next sensor frame is received and processed.

With this minimal representation of the ground truth obstacles, containing only the minimal information about the obstacles as required for safety-critical collision avoidance, we can now determine when detected obstacles are sufficiently detected to avoid collisions.
\subsection{Coverage}

Each detected obstacle that meets the distance criteria ($\S$\ref{sec:req_dist}) and falls within the direction of an obstacle can now be used to determine if a ground truth obstacle is sufficiently covered by detected obstacles to avoid collisions with the ground truth obstacle.
Detected obstacles are projected on the projection of the ground truth obstacle to determine what parts of the ground truth projection are covered by the detected obstacles, as shown in Figure~\ref{fig:req_examples}~(d).
The projected detection must provide enough information to avoid collisions for an obstacle to be considered as detected.
As with distance ($\S$\ref{sec:req_dist}) this coverage may not be perfect but should have bounded errors.
A limited proportion of the ground truth projection must be covered. This is equivalent to the traditionally used intersection over union (IOU) limits.
It should be noted that the error margin of coverage is less important than that of the distance.
The minimum distance always tracks the distance between the closest points between the AV and the obstacle, changing with each sensor input frame.

\subsection{Summary}

In this section we have discussed various features of obstacles a perception system may detect and detailed their relevance for safety-critical collision avoidance.
Many properties of obstacles that are considered critical parts of perception in AV
(\eg class, 3D dimensions, road sign information),
while indeed required for achieving the mission of autonomous driving, are not necessary for safety-critical collision avoidance.
In brief, an obstacle is considered detected for safety-critical obstacle avoidance if the following are accurately determined
\ca the distance of obstacle from AV, within bounded error margins.
\cb a line on the road that the obstacle makes unsafe for the AV to cross.
An obstacle detection system that reliably meets the above requirements can be used to detect safety-critical faults in perception by more complex perception systems, thus providing fault detection and collision avoidance system.
The fusion of verifiable algorithms with DNN and the control decisions based on the verifiable algorithms have additional challenges that will be addressed in future works.

\section{Parameters and Constraints}
\label{sec:paramconstraints}

\begin{table}[t]
    \centering
    \caption{\label{tab:symbol_summary}
        Symbols Summary}
    \begin{tabular}{c | p{0.46\columnwidth} | p{0.22\columnwidth}}
        \toprule
        Symbols 		& Description                                      & Example Value                             \\ \midrule
        $N$        		& Count of lasers in vertical array                & $64$                                      \\
        $\xi$      		& \lidar Beam Angles set                           & $[-2.4^o, ..., 17.6^o]$                   \\
        $H_L$      		& Height of \lidar sensor                          & $2.184~m$                                 \\
        $\Uppsi$   		& \lidar horizontal angle step                     & $0.136^o$                                 \\
        $h_o$      		& Height of obstacle                               & Variable                                  \\
        $\alpha_o$ 		& Inclination of obstacle from ground              & $[0^o, ..., 8^o]$~\cite{aashto2001policy} \\
        $D_{min}$  		& Distance to first \lidar return                  & $6.7~m$                                   \\
        $D_{max}$  		& Distance range of the sensor                     & $75~m$                                    \\
        $\alpha_{th}$ 	& Threshold angle for ground removal               & $10^o$                                    \\ \bottomrule
    \end{tabular}
\end{table}

\subsection{\lidar  Parameters}
\label{sub:lidar_params}
A rotating \lidar  contains several lasers stacked vertically. Each laser is given an index $ \mathbf{r} \in 1 ... N$, where $\mathbf{N}$ is the number of lasers.
Each laser has an inclination angle $\xi_i$ from the horizontal, positive below the horizontal, the set of which is represented by $\xi$.
The lowest laser below the horizontal is assigned index $r = 1$ and $r = N$ index represents the highest laser, usually inclined above the horizontal.
The \lidar is mounted on the vehicle at a height $\mathbf{H_L}$ above the wheelbase.
For rotating \lidar, the horizontal step angle $\mathbf{\Uppsi}$ is uniform and can be calculated as $360^o/SamplesPerRotation$.
The sensor is considered the origin point for the coordinate system.
At each rotation, a new column of N samples is recorded and indexed with $\mathbf{c}$.
The sensor returns a range image, a 2D depth image of range values $\mathbf{R_{r,c}}~\forall~r\in1...N,~c\in1...\Uppsi$.
While we assume the more prolific rotating \lidar, a solid state \lidar~\cite{roriz2021automotive} would have similar parameters.
Table~\ref{tab:symbol_summary} and Figure~\ref{fig:alpha_proof}, respectively, summarize and represent some of these parameters, using example values from real world dataset~\cite{sun2020scalability}.

\subsection{Constraints}
\label{sec:constraints}
We assume the following constraints for the  validity of the Detectability model in Section~\ref{sec:model}:

\texttt{\textbf{C1}:} All \lidar beams encountering solid obstacles, including ground, within the max range of the \lidar, return with strong enough intensity from the first obstacle they encounter so that the return is recorded. This assumption holds in real world except when;
\ca there are physical impediments in the air like dust, smoke, fog or rain; or,
\cb adversarial objects with extremely reflective, absorbent or transparent surfaces.

\texttt{\textbf{C2}:} The obstacle's width, projected on a plane perpendicular to the \lidar  beams falling on the obstacle, must be enough to have \lidar  points returned from it.
For example, a thin sheet aligned parallel to the \lidar  beam's direction can never be detected. According to the formula for circle arc length, a conservative bound can be established:
\begin{align}
    \text{Width Projected} &\geq \Uppsi D \pi / 180^o.
\end{align}
Since $\Uppsi$ is small, the arc length is approximately equal, though always greater than the chord length for the same points.
Being a small value, this width is not a prohibitive constraint, \eg using values from Table~\ref{tab:symbol_summary} the width constraint at max \lidar range of $75$ $m$ is just $17$ $cm$.

\texttt{\textbf{C3}:} At least one \lidar  point exists on the ground before the obstacle. This is an algorithmic constraint~\cite{bogoslavskyi17pfg}.
While we focus on the primary Top \lidar only in rest of this work, additional low height limited field of view (FOV) \lidar sensors can be used for obstacles closer than $D_{min}$, \eg Front, Rear, Side-Left and, Side-Right \lidar in Waymo Open Dataset~\cite{sun2020scalability}.

Furthermore, the following are assumed initially ($\S$\ref{sub:model_simple}) and relaxed in later sections:

\texttt{\textbf{A1}:} Obstacle touches the ground at $90^o$, \ie  $\alpha_o = 90^o$. Relaxed in Sections~\ref{sub:model_incline} and \ref{sub:model_gap}.

\texttt{\textbf{A2}:} There is no ground inclination change between the vehicle and the obstacle. Relaxed in Section~\ref{sub:model_inclined_ground}.

\texttt{\textbf{A3}:} No noise in detected range. Relaxed in Section~\ref{sub:model_noise}.

\section{Detectability Model}
\label{sec:model}

\begin{figure}[tp]
    \centering
        \centering
        \includegraphics[width=.7\linewidth,keepaspectratio]{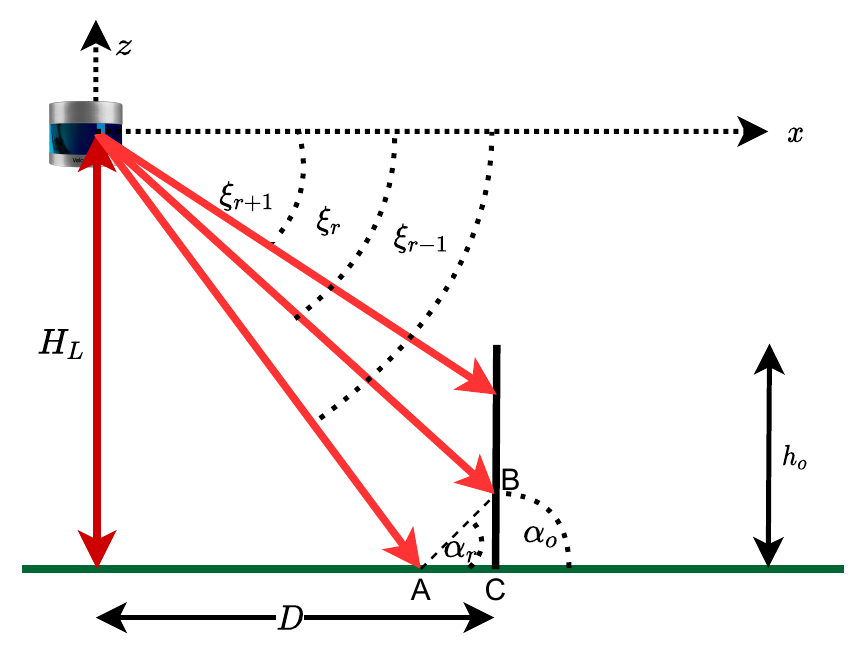}
        \caption{\label{fig:alpha_proof}
            A representation of point segmentation to ground \vs obstacle based on $\alpha$ thresholding when two points return from the obstacle.
        }
\end{figure}

A \textit{detectability model} describes, given an obstacle's properties, sensor parameters, and environment parameters, whether a given algorithm can detect the obstacle with a given set of algorithm parameters.
The detectability model allows the conversion of safety standards into algorithm and sensor parameter requirements as described with the example in Section~\ref{sec:motivation}.
We now develop the detectability model for obstacle detection using a \lidar.
We use Depth Clustering~\cite{bogoslavskyi16iros, bogoslavskyi17pfg} as the example algorithm, which uses depth discontinuity to segment \lidar points into ground \vs obstacles and then determine bounding boxes for obstacles.
Table~\ref{tab:symbol_summary} contains a summary of the symbols used in this section with example values from Waymo Open Dataset~\cite{sun2020scalability} or chosen defaults.

\begin{figure*}[tp]
    \centering
    \begin{minipage}[t]{0.32\textwidth}
        \centering
        \includegraphics[width=\linewidth,keepaspectratio]{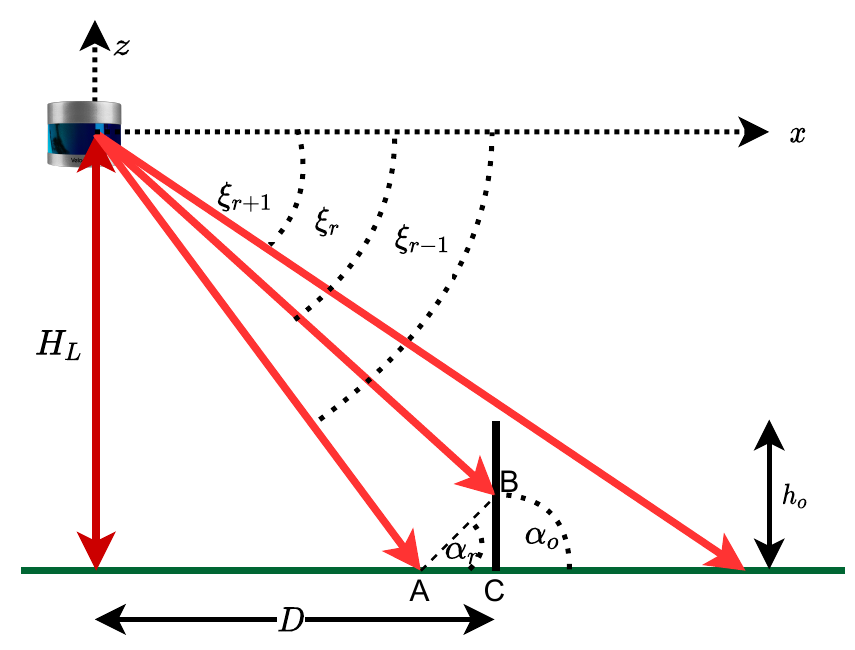}
        \caption{\label{fig:alpha_proof_one_point}
            Respresentation for $\alpha$ thresholding with only one point on obstacle.
        }
    \end{minipage}
    \hfill
    \begin{minipage}[t]{0.32\textwidth}
        \centering
        \includegraphics[width=\linewidth,keepaspectratio]{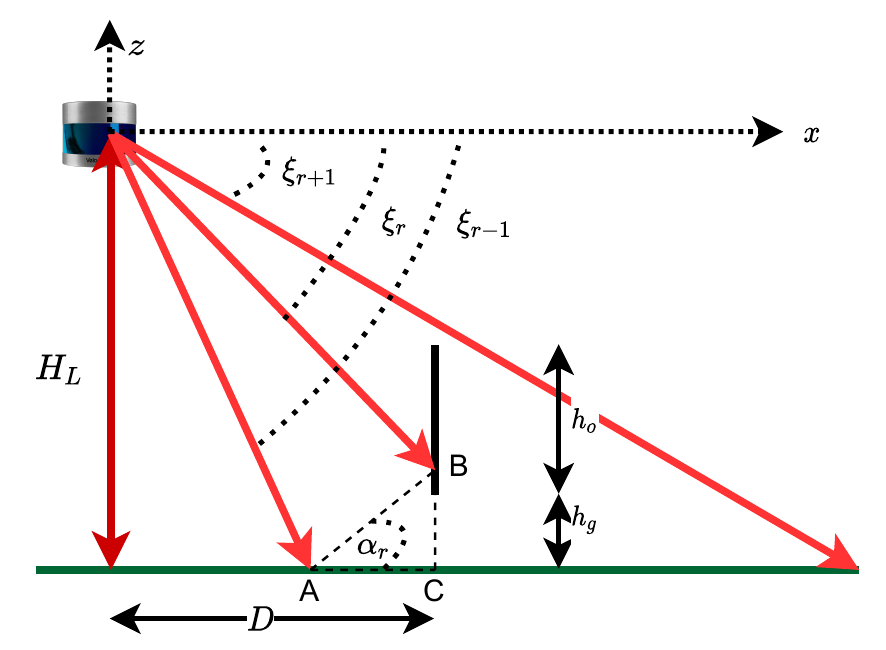}
        \caption{\label{fig:alpha_proof_gap}
            Obstacle above ground plane.
        }
    \end{minipage}
    \hfill
    \begin{minipage}[t]{0.32\textwidth}
        \centering
        \includegraphics[width=\linewidth,keepaspectratio]{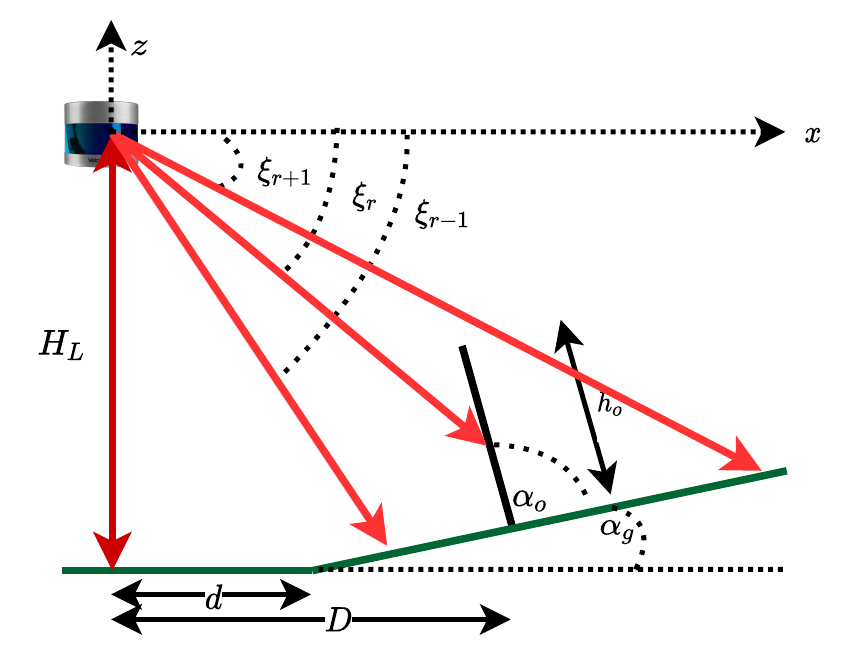}
        \caption{\label{fig:alpha_proof_ground_incline}
            Obstacle with ground inclination relative to the AV.
        }
    \end{minipage}
\end{figure*}

\subsection{Simple Model}
\label{sub:model_simple}

We start with a simple system model, as shown in Figures~\ref{fig:alpha_proof}~and~\ref{fig:alpha_proof_one_point}, with Assumptions \texttt{A1}, \texttt{A2} and \texttt{A3}.

\textit{Ground Removal:}
The primary part of the Depth Clustering algorithm is ground removal, \ie determining which range image points are on obstacles \vs ground.
FN errors can occur if points on obstacles are mistakenly considered part of the ground.
Ground removal is based on vertical depth discontinuity using the inclination angles $\alpha$.
Assuming the first point ($r=1$) to always be on the ground, a sharp change in $\alpha$ above a threshold  $\alpha_{th}$ indicates that a point is on an obstacle.
The first point must lie on the ground for the subsequent comparisons to be meaningful (\texttt{C3}).
Inclination angle is calculated from the range image, as designed by Bogoslavskyi and Stachniss~\cite{bogoslavskyi17pfg}, shown in Figure~\ref{fig:alpha_proof}, as
\begin{align}
    \alpha_{r,c} &= atan2(||BC||, ||AC||) = atan2(\Delta z, \Delta x) \label{eq:alpha} \\
    \Delta z_{r,c} &= |R_{r-1,c} sin \xi_{r-1} - R_{r,c} sin \xi_{r}| \\
    \Delta x_{r,c} &= |R_{r-1,c} cos \xi_{r-1} - R_{r,c} cos \xi_{r}| \\
    \Delta \alpha_{r,c} &= \left\{\begin{array}{lr}
        0^o,\text{if } r=1 \\%
        |\alpha_{r,c} - \alpha_{r-1,c}|,\text{otherwise}%
        \end{array}\right\} \label{eq:delta_alphs}
\end{align}

The point $r,c$ is considered to not be on the ground if, $R_{r-1,c}$ is on the ground, $\alpha_{th} < 45^o$, and
\begin{align}
    \Delta \alpha_{r,c} &>  \alpha_{th}, \label{eq:alpha_condition}
\end{align}

If $R_{r-1,c}$ is not on ground then $R_{r,c}$, is also considered to not be on ground.
The lowest row of points $R_{0,c}$ are assumed to be on the ground ($\S$\ref{sec:constraints} {\tt C3}).
The column index corresponds to a rotational position of the sensor.
For brevity, the column index is omitted for points in the same column only.

At horizontal distance $D$ from the sensor, the height of a \lidar beam $r$ can be calculated as:
\begin{equation}
    H_r(D, \xi_r, H_L) = H_L - D * tan(\xi_r).
    \label{eq:Hr1}
\end{equation}

For brevity, $H_r(D, \xi_r, H_L)$ is referred to as $H_r(D)$ from this point as $\xi_r$ and $H_L$ are constant once a \lidar sensor is chosen and placed on an AV. Let us define $r = \min \{i|H_i(D) > 0^o\}$, \ie $R_r$ is the first point on the obstacle.

With this background, we can determine the minimum height of an obstacle that can be detected when the obstacle is at a distance of $D$ from the sensor (Figure~\ref{fig:alpha_proof}). It is required to check three consecutive \lidar points to detect an obstacle, according to the ground removal algorithm presented in~\eqref{eq:alpha} through~\eqref{eq:alpha_condition}.

\begin{theorem}
\label{th:the1}
    The obstacle is detected at distance $D$, if and only if one of the following conditions are true:
    \begin{enumerate}
        \item $H_{r}(D) \leq h_o < H_{r+1}(D)$  AND \\ $\alpha_{th} < atan2(H_{r}(D), |D - \frac{H_L}{tan(\xi_{r-1})}|)$;
        \item $h_o \geq H_{r+1}(D)$.
    \end{enumerate}
\end{theorem}

\begin{proof}
\textit{Sufficiency:}
Consider case 2), \ie $h_o \geq H_{r+1}(D)$. Then the points $R_r$ and $R_{r+1}$ are on the obstacle, as shown in Figure~\ref{fig:alpha_proof}, and thus it holds that $\alpha_{r+1}=90^o$.
Since we have
\begin{align*}
    &\Delta \alpha_r = \alpha_r-0^o,&&
    \Delta \alpha_{r+1} = |90^o-\alpha_r|,
\end{align*}
where  $\alpha_{r-1}=0^o$ by the definition of index $r$, it follows that
\begin{align*}
    \max\{\Delta \alpha_r,\Delta \alpha_{r+1}\} \geq 45^o>\alpha_{th},
\end{align*}
which guarantees that the obstacle is detected at distance $D$.

Now consider case 1), \ie $H_{r}(D) \leq h_o < H_{r+1}(D)$ as shown in Figure~\ref{fig:alpha_proof_one_point}.
By the assumption, there exists an index $r-1$ that touches the ground.
Since $h_o \geq H_{r}(D)$, the angle $\alpha_{r-1}$ can be found by \eqref{eq:alpha} as follows:
\begin{align}
    \alpha_{r} &= atan2(H_{r}(D), |R_{r-1} cos (\xi_{r-1}) - R_{r} cos (\xi_{r})|)\nonumber\\
    &=atan2(H_{r}(D), |D - \frac{H_L}{tan(\xi_{r-1})}|),
    \label{eq:simple_cond}
\end{align}
where
\begin{align*}
    &R_{r-1}sin (\xi_{r-1}) = H_L,  &&R_{r} cos (\xi_{r}) = D.
\end{align*}

Since $\alpha_{r-1}=0^o$, the above equation and the angle condition in case 1) imply $\alpha_{th} < \Delta \alpha_{r}$, \ie the obstacle is detected at distance $D$.

\textit{Necessity:}
If the obstacle is detected at a distance $D$, then at least one point on the obstacle is at distance $D$, \ie $R_r$ is on the obstacle. This implies $h_o \geq H_{r}(D)$.
Furthermore, if the obstacle is detected, then
\begin{align}
    \Delta \alpha_{r} = \alpha_r-\alpha_{r-1} = \alpha_r > \alpha_{th}.
    \label{eq:alpha_th_cond}
\end{align}
There are two cases: $\alpha_{r+1}=90^o$ and $\alpha_{r+1} \neq 90^o$. If $\alpha_{r+1}=90^o$, then
$R_{r+1}$ is on the obstacle, and thus we have $h_o \geq H_{r+1}(D)$. This implies the case 2).
On the other hand, if $\alpha_{r+1} \neq 90^o$, then
the inequality in~\eqref{eq:alpha_th_cond} must hold, which renders the condition $\alpha_{th} \leq atan2(H_{r}(D), |D - \frac{H_L}{tan(\xi_{r-1})}|)$.
This implies case 1). This completes the proof.
\end{proof}

Theorem~\ref{th:the1} consists of two sets of conditions based on the obstacle's height. The second set indicates that if the height of the obstacle is sufficiently large, then there will be more than two \lidar points on the obstacle ($\alpha_{r+1}=90^o$). This allows us to detect the obstacle without further condition on the angle $\Delta \alpha_r$.
The first set is the case that, when the size of the obstacle is small, there is only one \lidar point on the obstacle, which requires us to have an additional condition for the angle $\Delta \alpha_r$.

The first condition set in Theorem~\ref{th:the1} implies that the minimum height to be detected at distance $D$ satisfies
\begin{align}
   &H_{r}(D) = h_o \nonumber\\
   &\alpha_{th} \leq atan2(h_o, |D - \frac{H_L}{tan(\xi_{r-1})}|)
   \label{eq:minimum_height},
\end{align}
which depends on the distance $D$ and threshold $\alpha_{th}$.

\subsection{Obstacle at Inclination}
\label{sub:model_incline}
\noindent
\textit{Assumptions:}
We maintain all assumptions from $\S$\ref{sub:model_simple} except part of \texttt{A1}, \ie that the obstacle surface inclination angle is $\alpha_o \neq 90^o$. We assume $\alpha_o > \alpha_{th}$, otherwise the obstacle cannot be detected.

In this case, if $R_{i}$ is at the end of the inclined obstacle, then its height at a distance $D$ is found by
\begin{align*}
    H_{i}(D) = h_o sin (\alpha_o) + h_o cos (\alpha_o) tan (\xi_{i}).
\end{align*}
This height can be used to determine whether there are more than two \lidar points on the obstacle. Further notice that the angle $\alpha_r$ is found by reduced height $\frac{H_r(D)}{tan(\xi_r)/tan(\alpha_o)+1}$ and increased width $D+\frac{H_r(D)}{tan(\xi_r)+tan(\alpha_o)}-\frac{H_L}{tan(\xi_{r-1})}$. This observations induce the following Corollary from Theorem~\ref{th:the1}.
\begin{corollary}
    The obstacle is detected at distance $D$, if and only if one of the following conditions are true:
    \begin{enumerate}
        \item $H_{r}(D) \leq h_o sin (\alpha_o) + h_o cos (\alpha_o) tan (\xi_{r+1}) < H_{r+1}(D)$ AND $\alpha_{th} < atan2(\frac{H_r(D)}{tan(\xi_r)/tan(\alpha_o)+1}, |D+\frac{H_r(D)}{tan(\xi_r)+tan(\alpha_o)}-\frac{H_L}{tan(\xi_{r-1})}|)$;
        \item $h_o sin (\alpha_o) + h_o cos (\alpha_o) tan (\xi_{r+1}) \geq H_{r+1}(D)$.
    \end{enumerate}
\label{cor1}
\end{corollary}

\subsection{Obstacle not Touching Ground}
\label{sub:model_gap}

\noindent
\textit{Assumptions:}
We maintain all assumptions in $\S$\ref{sub:model_simple} except \texttt{A1}, \ie that the obstacle surface no longer touches the ground. Instead, now the obstacle surface starts at height $h_g$ above the ground as shown in Figure~\ref{fig:alpha_proof_gap}.

We are interested in the first beam on the obstacle, where its index is defined by $r_g = \min \{i|H_i(D) > h_g\}$. Noticing the height of the obstacle tip is $h_o+h_g$, we can reformulate Theorem~\ref{th:the1} as the following corollary.
\begin{corollary}
    The obstacle is detected at distance $D$, if and only if one of the following conditions are true:
    \begin{enumerate}
        \item $H_{r_g}(D) \leq h_o+h_g < H_{r_g+1}(D)$ AND $\alpha_{th} < atan2(H_{r_g}(D), |D - \frac{H_L}{tan(\xi_{r_g-1})}|)$;
        \item $h_o+h_g \geq H_{r_g+1}(D, \xi, H_L)$.
    \end{enumerate}
\label{cor2}
\end{corollary}

The minimum detectable height for this case is found by
\begin{align*}
   &H_{r_g}(D) = h_o + h_g \nonumber\\
   &\alpha_{th} \leq atan2(h_o + h_g, |D - \frac{H_L}{tan(\xi_{r_g-1})}|).
\end{align*}

\newcommand{\maxfitplotsize}{.329}
\begin{figure*}[tp]
    \centering
    \subfloat[\centering \label{fig:hmin_simple} Simple Model ($\S$\ref{sub:model_simple})]{{\includegraphics[width=\maxfitplotsize\linewidth,keepaspectratio]{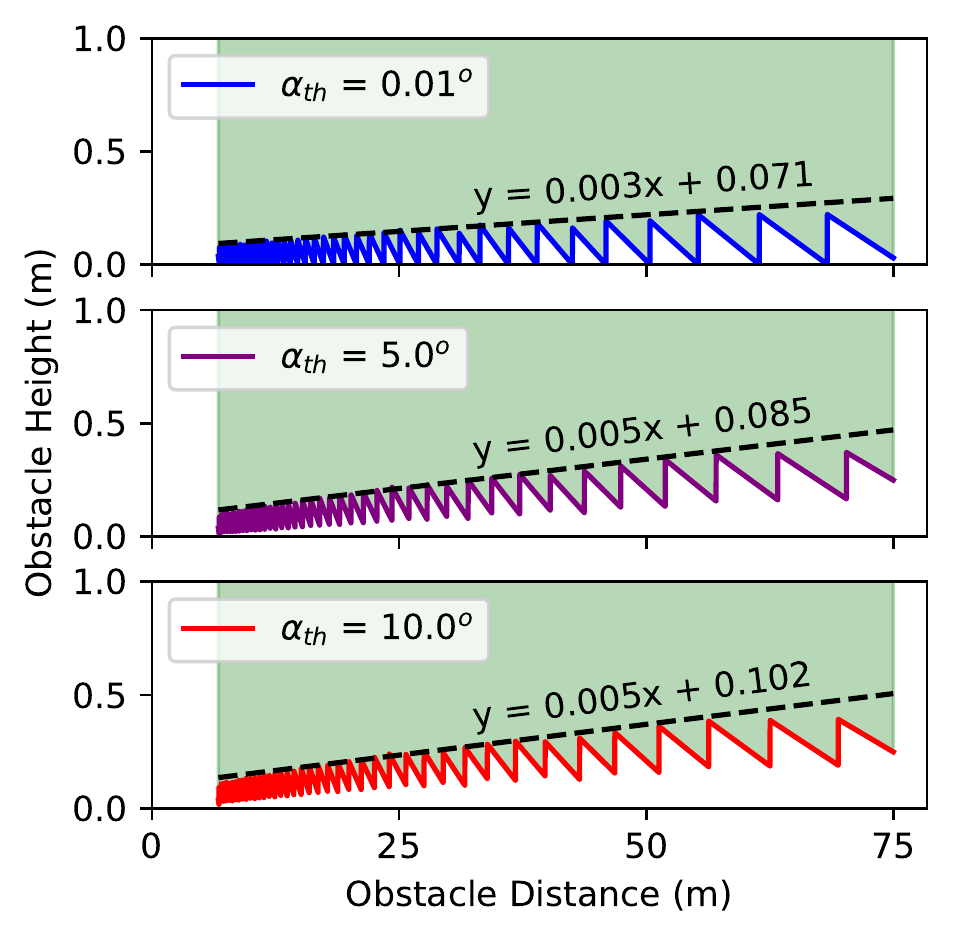}}}%
    \hfill
    \subfloat[\centering \label{fig:hmin_incline} $\alpha_o = 60^o$ ($\S$\ref{sub:model_incline})]{{\includegraphics[width=\maxfitplotsize\linewidth,keepaspectratio]{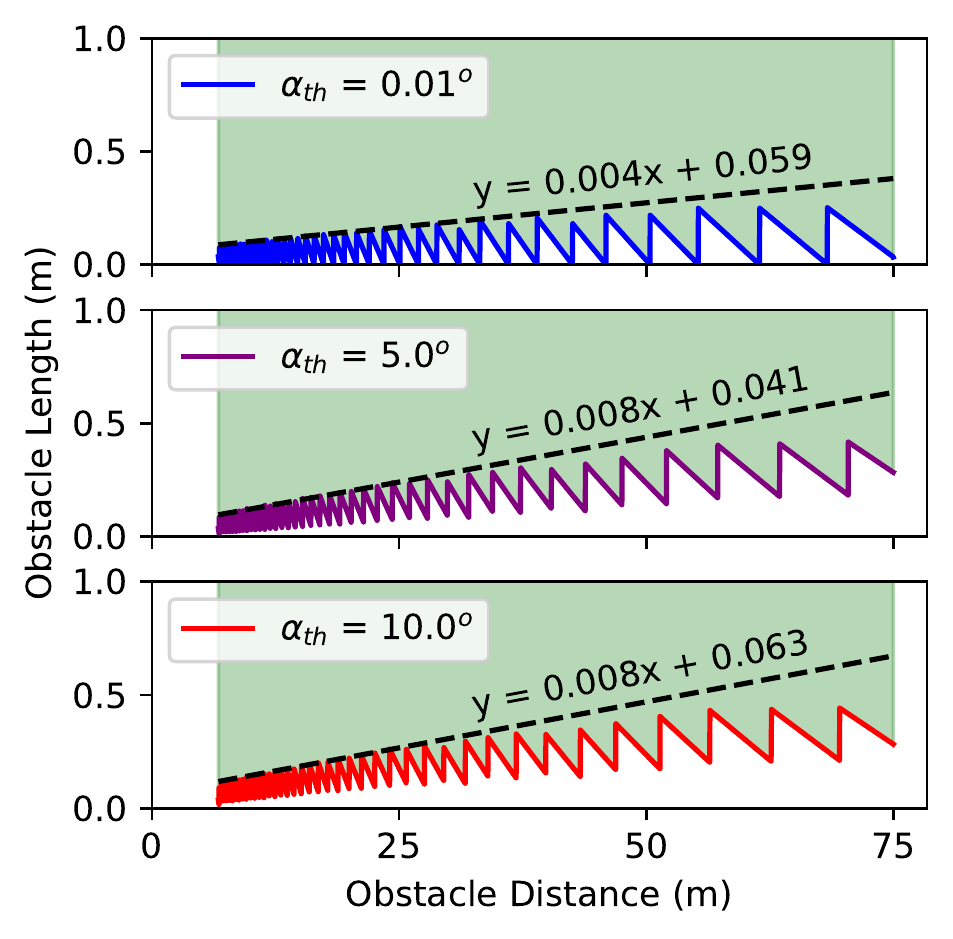}}}%
    \hfill
    \subfloat[\centering \label{fig:hmin_gap} $h_g = 0.271$ ($\S$\ref{sub:model_gap})]{{\includegraphics[width=\maxfitplotsize\linewidth,keepaspectratio]{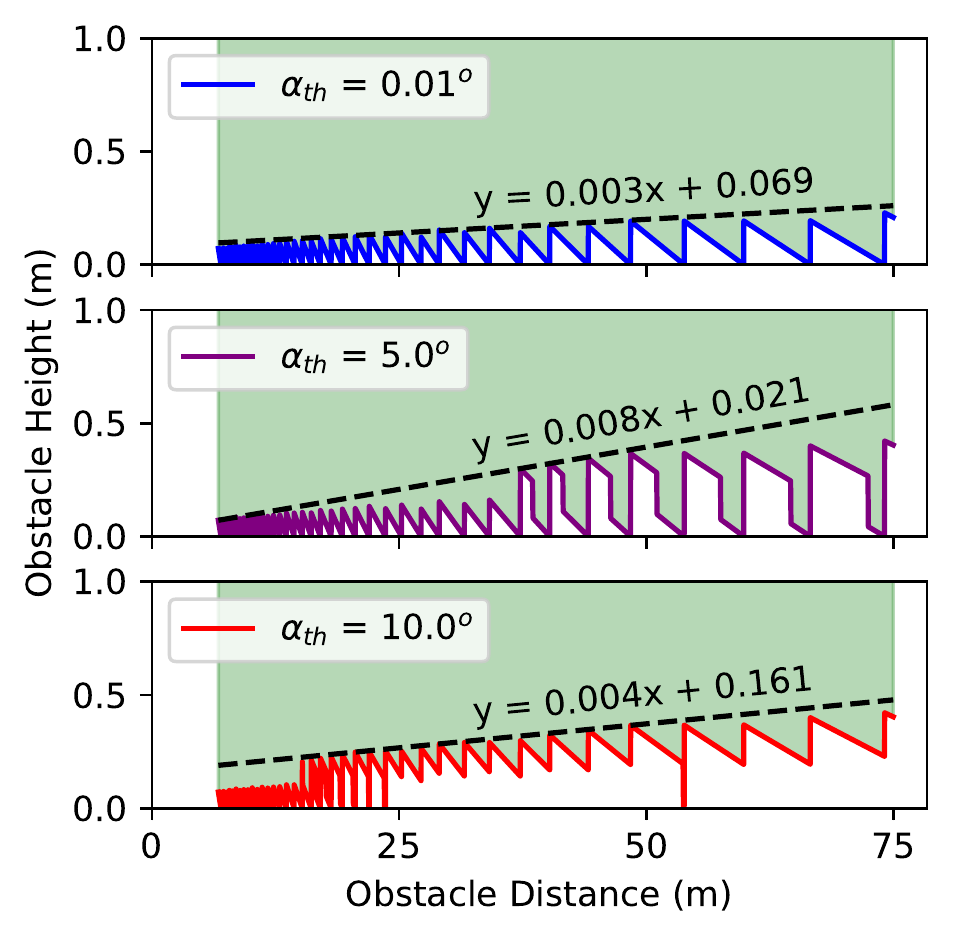}}}%
    \caption{
        Minimum detectable obstacle height (Y-Axis) at varying distances from the vehicle (X-Axis).
        Black dashed lines are a linear max fit for the plots, with corresponding equations.
        Shaded region signifies where the obstacle is always detected.}%
    \label{fig:hmin}%
\end{figure*}

\subsection{Inclined Ground}
\label{sub:model_inclined_ground}

\textit{Assumptions:}
We maintain all assumptions in $\S$\ref{sub:model_simple} except \texttt{A2}, \ie Ground inclination $\alpha_g$ could be non-zero (Figure~\ref{fig:alpha_proof_ground_incline}).

Ground inclination affects the detectability model only when relative inclination changes between the obstacle and the AV.
This section assumes that ground inclination starts before the obstacle position, \ie $d < D$.
There are two cases.

Case I: The inclination starts after the beam $R_{r-1}$, \ie
\begin{align}
	\frac{H_L}{tan(\xi_{r-1}) } \leq d \leq \frac{H_L}{tan(\xi_{r}) }.
\end{align}

We can extend Corollary~\ref{cor1} to find the detectability condition. The first beam touching the obstacle must be higher than the ground at distance $D$. Let us define
\begin{align*}
    r_f = \min \{i|H_i(D) > (D-d) tan (\alpha_g)\}.
\end{align*}

If $R_{r_f+1}$ is at the end of the obstacle, then
\begin{align*}
    H_{r_f+1}(D) &= h_o cos (\alpha_g) - h_o sin (\alpha_g) tan (\xi_{r_f+1})\nonumber\\
    &+ (D-d) tan (\alpha_g),
\end{align*}
which is the threshold to determine whether $R_{r_f+1}$ is on the obstacle.
Furthermore, $\alpha_{r_f}$ is found by increased height
$\frac{H_{r_f}(D)-(D-d) tan (\alpha_g)}{1-tan(\xi_{r_f})tan(\alpha_g)}+(D-d) tan (\alpha_g)$
and decreased width
$D+\frac{H_{r_f}(D)-(D-d) tan (\alpha_g)}{tan(\xi_{r_f})-cot(\alpha_g)}-\frac{H_L}{tan(\xi_{r_f-1})}$.
This observation induces the following Corollary.
\begin{corollary}
    The obstacle is detected at distance $D$, if and only if one of the following conditions are true:
    \begin{enumerate}
        \item $H_{r_f}(D) \leq
        h_o cos (\alpha_g) - h_o sin (\alpha_g) tan (\xi_{r_f+1})+ (D-d)tan (\alpha_g)
        < H_{r_f+1}(D)$ AND $\alpha_{th} < atan2(\frac{H_{r_f}(D)-(D-d) tan (\alpha_g)}{1-tan(\xi_{r_f})tan(\alpha_g)}+(D-d) tan (\alpha_g), |D+\frac{H_{r_f}(D)-(D-d) tan (\alpha_g)}{tan(\xi_{r_f})-cot(\alpha_g)}-\frac{H_L}{tan(\xi_{r_f-1})}|)$;
        \item $h_o cos (\alpha_g) - h_o sin (\alpha_g) tan (\xi_{r_f+1})+ (D-d)tan (\alpha_g) \geq H_{r_f+1}(D)$.
    \end{enumerate}
\label{cor3}

\end{corollary}

Case II:
The inclination starts before the beam $R_{r-1}$, \ie
\begin{align*}
	d < \frac{H_L}{tan(\xi_{r-1}) }.
\end{align*}

In this case, the beam point of $R_{r_f-1}$ land on different point from that of the case II. This decreases relative height $H_L-R_{r_f-1} sin (\xi_{r_f-1})$ and increases relative width $\frac{H_L}{tan (\xi_{r_f-1})}-R_{r_f-1} cos (\xi_{r_f-1})$ between $R_{r_f-1}$ and $R_{r_f}$.
This changes $\alpha_{r_f}$ found in Corollary~\ref{cor3}.
It is also worth to notice that if
\begin{align*}
	d < \frac{H_L}{tan(\xi_{r-2}) },
\end{align*}
then $\alpha_{r_f-1} = \alpha_g$, and thus we have $\Delta \alpha_r = \alpha_r - \alpha_g$.
Using these facts, one can find an extension of Corollary~\ref{cor3}.
\subsection{Noise}
\label{sub:model_noise}

\textit{Assumptions:}
We maintain all assumptions in $\S$\ref{sub:model_simple} except \texttt{A3}, \ie there exists a depth detection error.
Minor inaccuracies in depth detection can be a problem when finding $\Delta x, \Delta z$.
In particular, the $i^{th}$ beam returns noisy measurement $y_i = R_i \epsilon_i$~\cite{wang2019pseudo} instead of its ground truth distance $R_i$ from the sensor to the obstacle where $\epsilon_i \in [1-\epsilon,1+\epsilon]$ is unknown noise with a known sensor error bound $\epsilon$, \eg Velodyne HDL-64E S3 \lidar  has a range detection inaccuracy of $\pm2$ $cm$~\cite{velodynehdl64}.

Under mild assumptions, Theorems~\ref{th:the2} and~\ref{th:the3} provide a bound of the first angle $\alpha_r$ on the obstacle, and a bound of the angle $\alpha_{i}$ between two consecutive points on the obstacle.

\begin{theorem}
Assume
$y_{r}cos(\xi_r) \geq y_{r-1}cos(\xi_{r-1})$ and $y_{r}sin(\xi_r) \leq y_{r-1}sin(\xi_{r-1})$. Then, the angle $\alpha_r$ is lower bounded by
\begin{alignat*}{2}
\alpha_r &\geq atan2(&&R_{r-1}(1-\epsilon) sin(\xi_{r-1})-R_r(1+\epsilon) sin(\xi_r),\nonumber\\
    & && R_r(1+\epsilon) cos(\xi_r)-R_{r-1}(1-\epsilon) cos(\xi_{r-1}))
\end{alignat*}
and upper bounded by
\begin{alignat*}{2}
\alpha_r &\leq atan2(&&R_{r-1}(1+\epsilon) sin(\xi_{r-1})-R_r(1-\epsilon) sin(\xi_r),\nonumber\\
    & && R_r(1-\epsilon) cos(\xi_r)-R_{r-1}(1+\epsilon) cos(\xi_{r-1})).
\end{alignat*}
\label{th:the2}
\end{theorem}
\begin{proof}
Angle $\alpha_r$ is found by \eqref{eq:alpha}. The assumptions $y_{r}cos(\xi_r) \geq y_{r-1}cos(\xi_{r-1})$ and $y_{r}sin(\xi_r) \leq y_{r-1}sin(\xi_{r-1})$ imply that
\begin{alignat*}{2}
\alpha_r &= atan2(&&|y_r sin(\xi_r)-y_{r-1} sin(\xi_{r-1})|,\nonumber\\
    & && |y_r cos(\xi_r)-y_{r-1} cos(\xi_{r-1})|)\nonumber\\
    &= atan2(&&y_{r-1} sin(\xi_{r-1})-y_r sin(\xi_r),\nonumber\\
    & && y_r cos(\xi_r)-y_{r-1} cos(\xi_{r-1})).
\end{alignat*}
Considering the fact that $tan$ is a strictly increasing function in the domain $[0^o,90^o)$, we can find the lower bound and upper bounds presented in the theorem statement.
\end{proof}

\begin{theorem}
Assume that two consecutive points $R_{i-1}$ and $R_{i}$ land on the obstacle, and that
$y_{i}sin(\xi_i) \leq y_{i-1}sin(\xi_{i-1})$. Then, the angle $\alpha_i$ is lower bounded by
\begin{alignat*}{2}
\alpha_i &\geq atan2(&&R_i(1+\epsilon) sin(\xi_i)-R_{i-1}(1-\epsilon) sin(\xi_{i-1}),\nonumber\\
    & && R_i(1+\epsilon) cos(\xi_i)-R_{i-1}(1-\epsilon) cos(\xi_{i-1}))
\end{alignat*}
and upper bounded by $90^o$.
\label{th:the3}
\end{theorem}
\begin{proof}
The function $atan2$ is upper bounded by $90^o$ in all domains.
Similarly, the lower bound can be found with the proof of Theorem~\ref{th:the2}. Details omitted due to the space limit.
\end{proof}

The assumptions $y_{r}cos(\xi_r) \geq y_{r-1}cos(\xi_{r-1})$ and $y_{r}sin(\xi_r) \leq y_{r-1}sin(\xi_{r-1})$ are mild
because $R_{r}cos(\xi_r) \geq R_{r-1}cos(\xi_{r-1})$ and $R_{r}sin(\xi_r) \leq R_{r-1}sin(\xi_{r-1})$ hold for most of the cases, and $\epsilon_i$ is around $1$.
For the same reason, the assumption $y_{i}sin(\xi_i) \leq y_{i-1}sin(\xi_{i-1})$ is mild as well.

\subsection{Summary}
\label{sub:model_summary}

Using $\xi$ and $H_L$ from the Waymo Open Dataset~\cite{sun2020scalability}, \eqref{eq:minimum_height} yields Figure~\ref{fig:hmin_simple} for various $\alpha_{th}$ and $D$ (X-Axis).
Figure~\ref{fig:hmin_incline} shows the same when $\alpha_o = 60^o$ as in Section~\ref{sub:model_incline}.
Figure~\ref{fig:hmin_gap} is based on Section~\ref{sub:model_gap}, assuming obstacle is $0.271$ $m$ above the ground, based on ground clearance of a sedan~\cite{matlab_sedan}.
The varying nature of the colored plot is due to the discrete nature of \lidar beams.
To get a usable bound, like \eqref{eq:ideal}, we determine a linear max curve fit for each case, where the dashed line is always greater than or equal to the original plot. Shaded region is where the obstacle is always detected.

In this section, we have developed detectability model for the Depth Clustering algorithm.
We find that not only is such analysis possible, but it also yields human perceptible bounds.
This work stands as the first step in establishing the verifiability of classical obstacle detection algorithms and their use in AV for safety-critical obstacle avoidance.

\section{Evaluation}
\label{sec:eval}

\begin{figure}[tp]
    \centering
    \subfloat[\centering \label{fig:gt_too_large}Example showing 3D ground truth box extending too far beyond the obstacle.]
    {{\fbox{\centering\includegraphics[width=.8\linewidth,keepaspectratio]{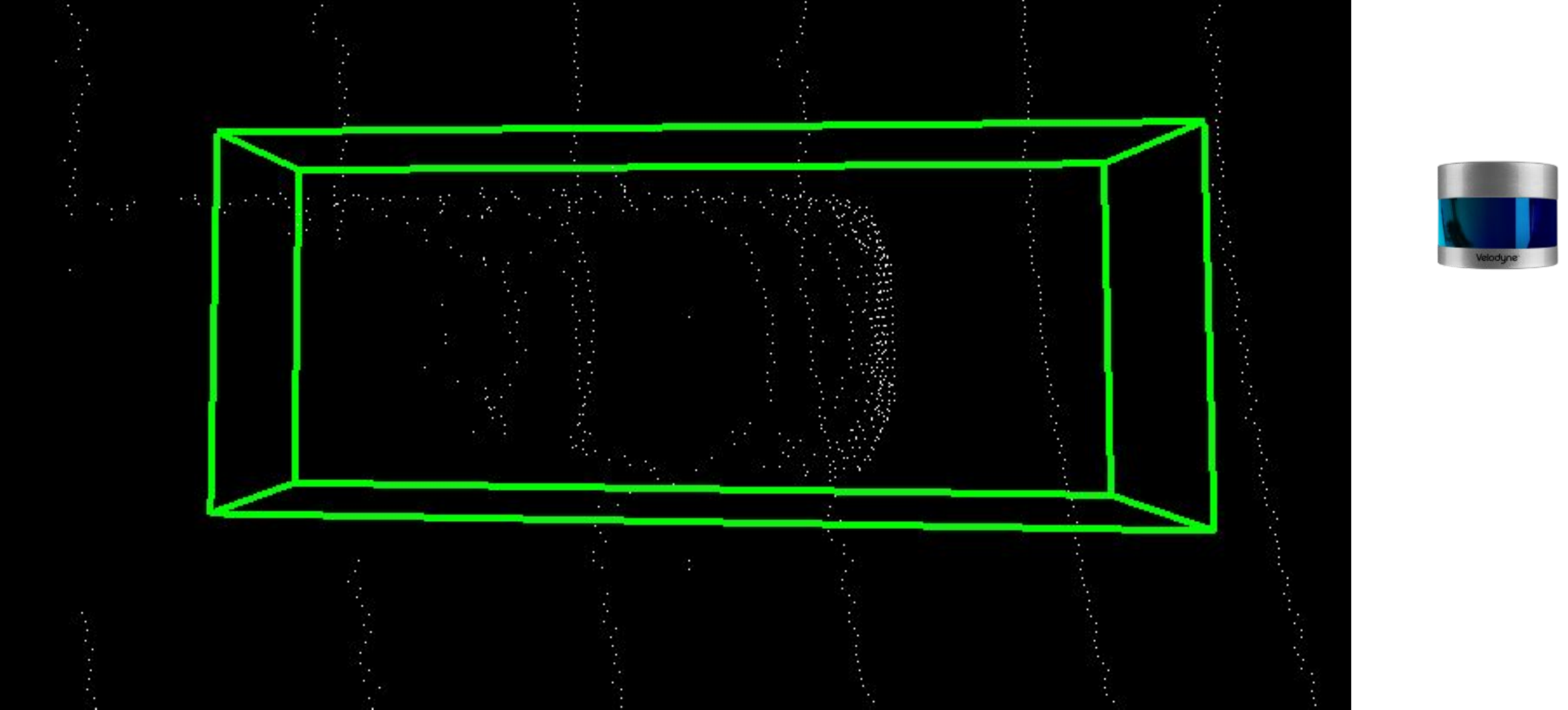}}}}%
    \hfill
    \subfloat[\centering \label{fig:coverage_example} GT extra size causing coverage (73.5\% here) to fall just below the threshold (75\%).]
    {{\fbox{\includegraphics[width=.74\linewidth,keepaspectratio]{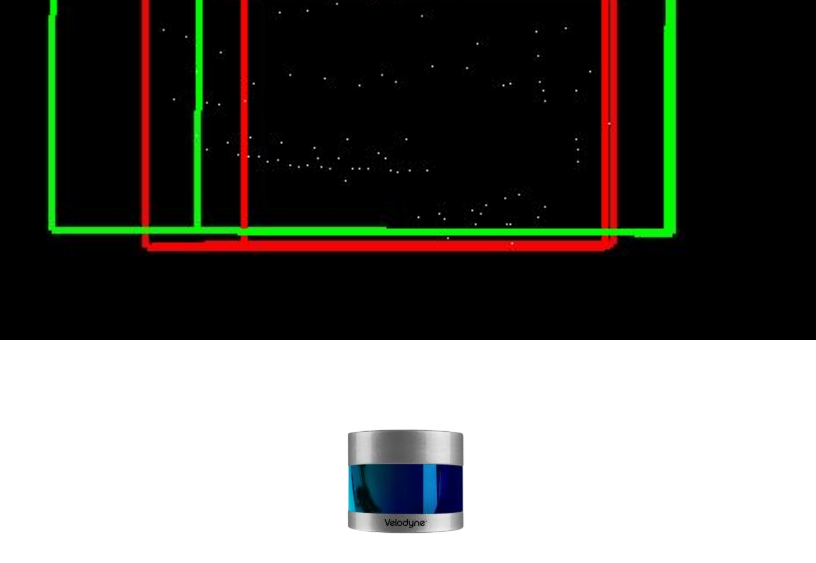}}}}%
    \hfill
    \subfloat[\centering \label{fig:dist_threshold} TP as the detection includes the visible edge of the vehicle towards the AV.]
    {{\fbox{\includegraphics[width=.7\linewidth,keepaspectratio]{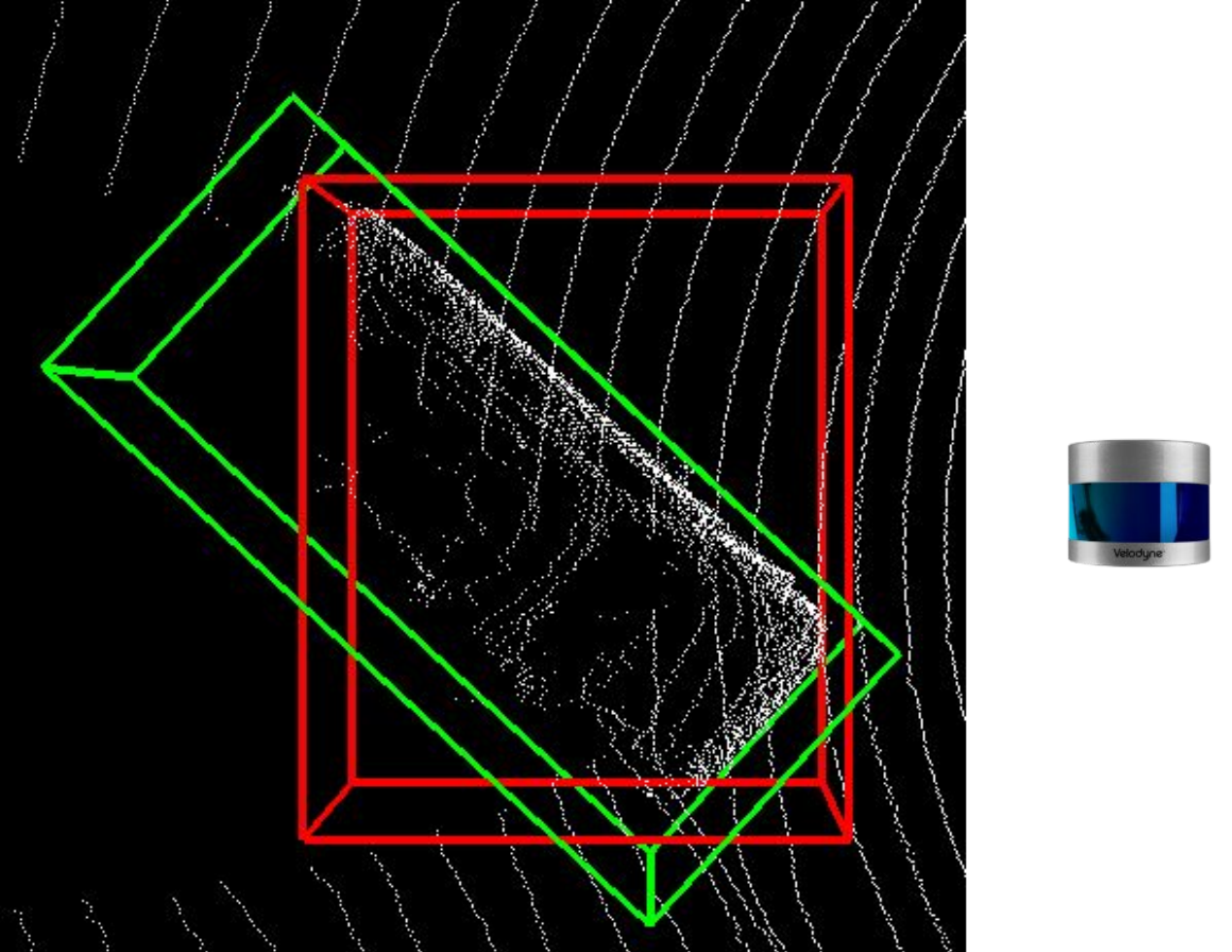}}}}%
    \hfill
    \subfloat[\centering \label{fig:oversegment} Example of Oversegmentation]
    {{\fbox{\includegraphics[width=.7\linewidth,keepaspectratio]{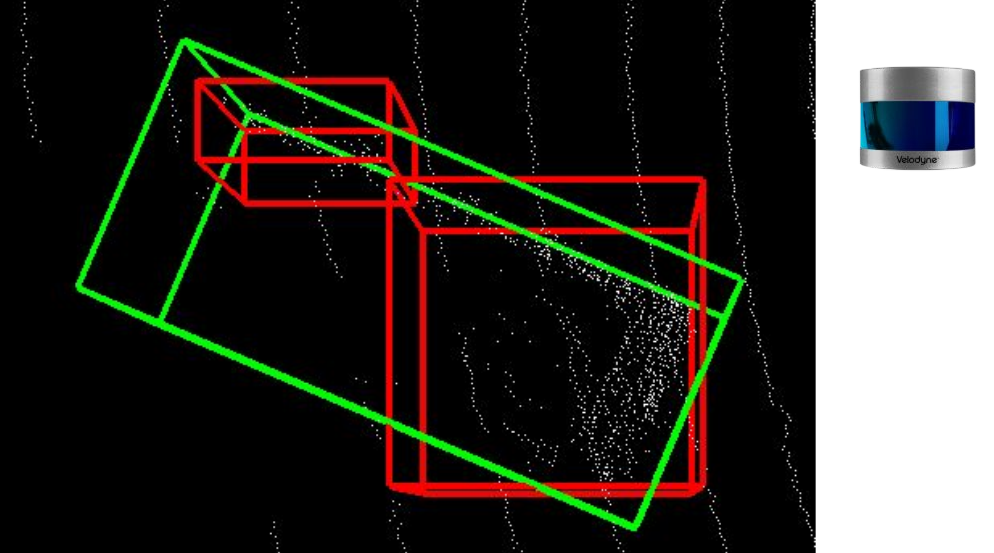}}}}%
    \hfill
    \caption{
        \label{fig:gt_issues}
        Examples from manual inspections of FN candidates.
        The white points are in 3D space as returned by \lidar \xspace \ie the point cloud.
        Ground truth 3D labels are drawn as green bounding boxes, while detection bounding boxes are red.
        The point cloud is shown from a top view, and the \lidar direction is indicated with the sensor image.
    }
\end{figure}

\begin{table}[tp]
    \centering
    \caption{\label{tab:eval_clear}
        Results for False Negative (FN) Evaluation}
    \begin{tabular}{r|l}
        \toprule

        Count & Category \\

        \midrule

        98224    	& Total obstacles (Vehicles, Pedestrian, Cyclist, Unknown) \\
        93030    	& Total without obstacles closer than $D_{min}$ (\texttt{C3}  $\S$\ref{sec:constraints}) \\
        7565     	& Obstacles that pose collision risk ($\S$\ref{sub:rrr})~\cite{9460196}  \\
        7402		& Detected meeting minimally sufficient requirements ($\S$\ref{sec:minimal_ca}) \\
        153    		& Ground Truth larger than obstacle \\
        10    		& Oversegmentation \\
        0		    & Remaining FN count at 75\% coverage.  \\
        \bottomrule
    \end{tabular}
\end{table}

In this section, we evaluate the obstacle detection algorithm using real-world sensor data from Waymo Open Dataset~\cite{sun2020scalability}.
\label{sec:eval_fn}
We first determine if the algorithm meets the safety-critical requirements proposed in this work ($\S$\ref{sec:minimal_ca}).
For obstacle existence fault or FN evaluation we randomly select 16 clear weather scenes from the 202 scenes in the validation dataset.
The random downsampling was necessitated by the manual effort involved in analyzing FN candidates.
The scenes included various scenarios, including heavy to low traffic, residential with pedestrians, city and highway driving.

Table~\ref{tab:symbol_summary} summarizes the parameters used. $\alpha_{th}$ was set at $10^o$.
Distance overestimation error was bounded to $10$ $cm + 5\%$ of actual distance. The constant was chosen as a small value to keep the error bound low at close distances.
However, since the depth perception accuracy reduces with distance an additional percentage-based threshold allows accommodation of sensor limits while having a low impact.
A limited coverage threshold of $75\%$ was used as a threshold for True Positive (TP) detection, chosen to parallel the strict metric for IOU from COCO Dataset~\cite{coco}.
All FN candidates were analyzed manually using a visualizer provided by Bogoslavskyi and Stachniss~\cite{bogoslavskyi17pfg}.
We enhanced the visualizer
to aid the manual inspection.
Table~\ref{tab:eval_clear} summarizes the results.

\subsubsection{GT Counts}
We first determine the total number of GT in the included scenes. All classes, excluding road signs, were counted to yield a total of \textbf{98224}.
Obstacles closer than first beam ($R_{1,c}$) on the ground, \ie closer than $D_{min}$ are removed as per \texttt{C3} in Section~\ref{sec:constraints}, reducing the count of GT to \textbf{93030}.
Obstacles that do not pose a risk of collision are also removed~($\S$\ref{sub:rrr}), leaving
\textbf{7565} obstacles.

\subsubsection{Automated Evaluation}
We run the detection algorithm and evaluate results based on requirements described in Section~\ref{sec:minimal_ca}. \textbf{7402} True Positives and \textbf{163} FN candidates were found. The FN candidates were then manually inspected.

\subsubsection{GT Dimension Error}
\label{sec:gt_execption}

The most common reason for erroneous FN indication were inaccurate GT labels.
This inaccuracy of GT was determined based on visual inspection of point clouds.
The GT were visibly larger or offset from the actual obstacle.
Figures~\ref{fig:gt_too_large}~and~\ref{fig:coverage_example} show such examples.
Drawn as per the provided GT, the green box is clearly either larger or offset from the contained obstacle.
Similarly, as shown in Figure~\ref{fig:coverage_example}, the GT inaccuracy is enough to bring the coverage below the $75\%$ threshold.
In some cases, the error is small, \eg Figure~\ref{fig:dist_threshold}. However when close to the AV, the small GT error can still be larger than the distance overestimation error allowed. We consider these detections as TP after ascertaining that the detection bounding box meets the requirements.
A total of \textbf{153} FN candidates were found to fall in this category.

\subsubsection{Oversegmentation}

Figure~\ref{fig:oversegment} shows a case where the obstacle was adequately detected but segmented into more than one bounding box. Since the second bounding box did not meet the distance threshold, the automated analysis ignored it.
However, given the presence of both bounding boxes, we argue that this detection should be considered True Positive.
The points on the obstacle were not erroneously considered to be drivable ground.
\textbf{10 }instances of this scenario on the same vehicle were found in consecutive frames.

\subsubsection{Obstacle Existence Fault}
No FN, \ie obstacle existence faults were found.
This is not surprising given the low minimum height bounds determined in Section~\ref{sec:model} and Figure~\ref{fig:hmin}.
The max curve fits in Figures~\ref{fig:hmin_simple},~\ref{fig:hmin_incline}~and~\ref{fig:hmin_gap} are conservative linear approximations.
Ignoring the max curve fits, the actual bounds in Figures~\ref{fig:hmin_simple},~\ref{fig:hmin_incline}~and~\ref{fig:hmin_gap} were less than $0.5$ $m$ in all cases within the sensor range of $75$ $m$.

\section{Discussion}
\label{sec:discussion}

{\it \textbf{Generality}}:
The safety-critical requirements for collision avoidance are applicable to all ground based autonomous vehicles.
The \lidar parameters used apply broadly to all \lidar sensors.
The constraints are also generally applicable, except \texttt{C3} which is based on the algorithm.
Other verifiable algorithms may have different algorithmic constraints.
The detectability model in this work is specific to the chosen algorithm.
However, the analysis shows that it is indeed possible to derive strict human perceptible bounds on the detectability of these algorithms and serves as guidance for verification of similar algorithms.

{\it \textbf{Human Comprehensibility}}:
As shown in this work, safety standards, policies and limitations of the example algorithm are in human comprehensible definitions.
This makes such policies realistically implementable and enforceable.
Human comprehensible limitations also implicitly protect against adversarial objects~\cite{cao2019adversarial,tu2020physically}.
This is in contrast to the machine learning based solutions where requirements, faults and adversarial perturbations are not always expressible in human perceptible forms~\cite{huang2017safety}.

{\it \textbf{Requirements}}:
Section~\ref{sec:minimal_ca} establishes a minimal set of requirements for collision avoidance in AV.
However, additional features of obstacles, if determined in a verifiable manner, can improve the safety envelop, reducing overly conservative behaviors.
For example, the collision risk model uses obstacle and AV velocity to determine if there is a potential risk of collision.
If velocities cannot be determined reliably, a conservative default high velocity  threshold must be used.

{\it \textbf{Adversarial Objects}}:
The failure modes for analyzable algorithms are well defined and expressible in limitations like minimum height and slopes from ground.
Whereas DNN based detectors can have varied failure modes, including fully or partially designed adversarial objects~\cite{cao2019adversarial,tu2020physically,abdelfattah2021adversarial,zhu2021can}.
The difference in failure modes suggests than an ensemble of the two would be robust against attacks using adversarial designed objects.

\section{Future Work}
\label{sec:future_work}

{\it \textbf{Integration}}:
In this work we focus on \textit{detection} of faults in obstacle existence detection. However, the \textit{reaction} to them, \ie fault handling has its own challenges.
The integration of verifiable algorithms within existing AV pipelines and fault handling built upon it are the focus of our future research.

{\it \textbf{Precision Improvements}}:
The complete Depth Clustering algorithm includes methods for improving the detection precision.
Future works will address the expansion of the detectability model to include these methods.

{\it \textbf{Verifiable Algorithms}}:
Despite the necessity of DNN in AV, this work shows that there is a role for analyzable algorithms.
Therefore further research is warranted to improve such physical model backed verifiable algorithms.
Improvements like lower physical limitations bounds and lower FP rates within those bounds would make these algorithms have reduced impact to the performance of the AV while maintaining the same safety guarantees.

{\it \textbf{Weather}}:
Impediments like rain, fog, dust or smoke distort the \lidar returns and can result in \lidar beams  returning with low enough intensity to not be recorded or causing a false early return~\cite{heinzler2019weather,wallace2020full}, \ie violating \texttt{C1} ($\S$\ref{sec:constraints}).
Detectability in the presence of such faults is an avenue for future research.

\section{Conclusion}
\label{sec:conclusion}

This paper identifies requirements for safety-critical obstacle detection and presents a safety analysis of an obstacle detection algorithm.
The results encourage a thorough separation of mission and safety-critical requirements in autonomous vehicles.
Furthermore, verifiable algorithms could fulfill the critical safety requirements offloading that responsibility from DNN based solutions that remain inherently unverifiable.
\section*{Acknowledgment}

The material presented in this paper is based upon work supported by
the National Science Foundation (NSF) under grant no. CNS 1932529, ECCS 2020289,
the Air Force Office of Scientific Research (AFOSR) under grant no. \#FA9550-21-1-0411,
the National Aeronautics and Space Administration (NASA) under grant no. 80NSSC20M0229, AWD-000577-G1,
and University of Illinois Urbana-Champaign under grant no. STII-21-06.
Any opinions, findings, and conclusions or recommendations expressed in
this publication are those of the authors and don't necessarily
reflect the views of the sponsors.

\bibliographystyle{IEEEtran}
\bibliography{ref}

\end{document}